\documentclass[sigconf]{acmart}
\AtBeginDocument{%
}
\copyrightyear{2024}
\acmYear{2024}
\setcopyright{rightsretained}
\acmConference[KDD '24]{Proceedings of the 30th ACM SIGKDD Conference on Knowledge Discovery and Data Mining}{August 25--29, 2024}{Barcelona, Spain}
\acmBooktitle{Proceedings of the 30th ACM SIGKDD Conference on Knowledge Discovery and Data Mining (KDD '24), August 25--29, 2024, Barcelona, Spain}\acmDOI{10.1145/3637528.3671962}
\acmISBN{979-8-4007-0490-1/24/08}
\makeatletter
\gdef\@copyrightpermission{
\begin{minipage}{0.3\columnwidth}
\href{https://creativecommons.org/licenses/by/4.0/}
{\includegraphics[width=0.90\textwidth]{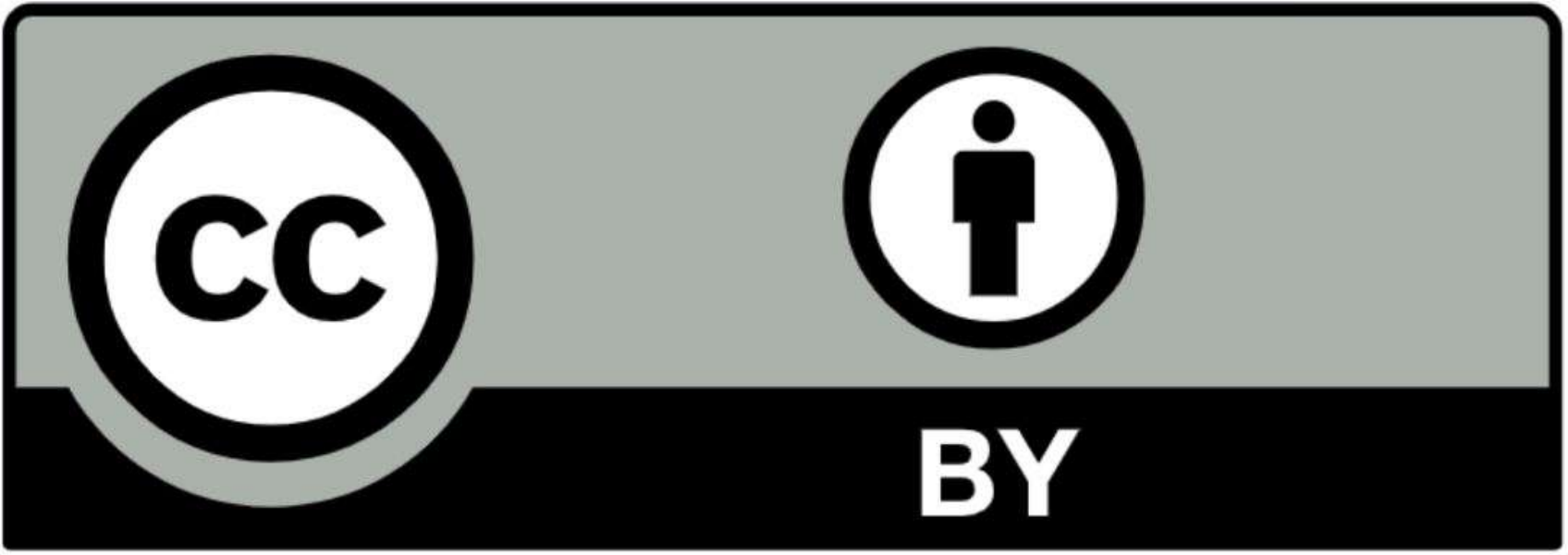}}
\end{minipage}\hfill
\begin{minipage}{0.7\columnwidth}
\href{https://creativecommons.org/licenses/by/4.0/}{This work is licensed under a Creative Commons Attribution International 4.0 License.}
\end{minipage}
\vspace{5pt}
}
\makeatother

\usepackage{amsfonts}
\usepackage{graphicx}
\usepackage{subcaption}
\usepackage{amsmath}
\usepackage{multirow}
\usepackage[linesnumbered,ruled,vlined]{algorithm2e}
\usepackage{tabularx}
\usepackage{enumitem}
\usepackage{pbox}
\usepackage{float}
\usepackage{hyperref}
\usepackage{setspace}
\usepackage{lipsum}
\usepackage{relsize}
\usepackage{tikz-cd}
\usepackage{mathtools}
\usepackage{latexsym}
\usepackage{color, colortbl}
\usepackage{multicol}
\usepackage{balance} 

\definecolor{grey}{rgb}{0.898,0.898,0.898}
\definecolor{navy}{rgb}{0.725,0.792,0.995}
\SetKwInput{KwInput}{Input}                
\SetKwInput{KwOutput}{Output}              

\SetCommentSty{mycommfont}

\newtheorem{prop}{Proposition}

\hypersetup{
    colorlinks=true,
    linkcolor=black,
    filecolor=magenta,      
    urlcolor=cyan,
    pdftitle={Overleaf Example},
    pdfpagemode=FullScreen,
    }

\begin{document}

\title{Self-Explainable Temporal Graph Networks based on \\ Graph Information Bottleneck}

\author{Sangwoo Seo}
\affiliation{%
  \institution{KAIST}
  \city{Daejeon}
  \country{Republic of Korea}}
\email{sangwooseo@kaist.ac.kr}

\author{Sungwon Kim}
\affiliation{%
  \institution{KAIST}
  \city{Daejeon}
  \country{Republic of Korea}}
\email{swkim@kaist.ac.kr}

\author{Jihyeong Jung}
\affiliation{%
  \institution{KAIST}
  \city{Daejeon}
  \country{Republic of Korea}}
\email{mjajthh1@kaist.ac.kr}

\author{Yoonho Lee}
\affiliation{%
  \institution{KAIST}
  \city{Daejeon}
  \country{Republic of Korea}}
\email{sml0399benbm@kaist.ac.kr}

\author{Chanyoung Park}
\authornote{Corresponding author.}
\affiliation{%
  \institution{KAIST}
  \city{Daejeon}
  \country{Republic of Korea}}
\email{cy.park@kaist.ac.kr}

\renewcommand{\shortauthors}{Sangwoo Seo et al.}

\begin{abstract}
Temporal Graph Neural Networks (TGNN) have the ability to capture both the graph topology and dynamic dependencies of interactions within a graph over time. 
There has been a growing need to explain the predictions of TGNN models due to the difficulty in identifying how past events influence their predictions.
Since the explanation model for a static graph cannot be readily applied to temporal graphs due to its inability to capture temporal dependencies, recent studies proposed explanation models for temporal graphs.
However, existing explanation models for temporal graphs rely on \textit{post-hoc} explanations, requiring separate models for prediction and explanation, which is limited in two aspects: efficiency and accuracy of explanation.
In this work, we propose a novel \textit{built-in} explanation framework for temporal graphs, called Self-Explainable Temporal Graph Networks based on Graph Information Bottleneck (TGIB). 
TGIB provides explanations for event occurrences by introducing stochasticity in each temporal event based on the Information Bottleneck theory.
Experimental results demonstrate the superiority of TGIB in terms of both the link prediction performance and explainability compared to state-of-the-art methods.
This is the first work that simultaneously performs prediction and explanation for temporal graphs in an end-to-end manner. The source code of TGIB is available at \url{https://github.com/sang-woo-seo/TGIB}.
\end{abstract}


\maketitle

\section{Introduction}\label{sec:intro}
Temporal Graph Neural Networks (TGNN) possess the capability to capture interactions over time in graph-structured data and demonstrate high utility in areas such as user-item interaction in e-commerce~\cite{li2021happens} and friend relationships in social networks~\cite{pereira2018analyzing, gelardi2021temporal}.
TGNNs incorporate both temporal dynamics and graph topology in their approach and focus on learning time-dependent node representation to predict future evolutions~\cite{xu2020inductive, rossi2020temporal, cong2023we}.
However, TGNN models are considered as black boxes with limited transparency due to the inability to discern how past events influence outcomes.
Offering insights based on the logic of predictions in TGNN contributes to an improved comprehension of the model's decision-making and provides rationality for predictions.
Explainability for TGNN can be applied in high-risk situations such as healthcare forecasting~\cite{li2021explaining, amann2020explainability} and fraud detection~\cite{sinanc2021explainable, psychoula2021explainable} to enhance the model's reliability, and assists in examining and mitigating issues related to privacy, fairness, and safety in real-world applications~\cite{doshi2017towards}.

Explainability aims to provide users with evidence within the data that influenced a model prediction. With the emerging necessity for explainability, explanation models for static graphs have been actively studied in recent times~\cite{ying2019gnnexplainer, luo2020parameterized, lee2023shift}.  
These models induce perturbations in the input of the model to detect the nodes and edges, which significantly impact the final prediction.
However, these models cannot be easily generalized to temporal graphs due to the high dynamicity of temporal graphs.
Specifically, explanation models for static graphs cannot capture the graph topology that is dynamic in nature
\cite{you2022roland, xu2020inductive}.  
For example, in a temporal graph, multiple events may occur over time between the same pair of nodes, and these events may have 
different importance depending on when the event occurred.
In other words, events that occurred a long time ago may have less influence on the current event compared to events that occurred more recently.


Recently, T-GNNExplainer~\cite{xia2022explaining} attempted to explain the model predictions on temporal graphs. Specifically, T-GNNExplainer consists of a navigator that learns inductive relationships between target events (those that are to be predicted) and candidate events (those that may serve as reasons for the prediction), and an explorer that explores the optimal combinations of candidate events for each target event based on Monte Carlo Tree Search (MCTS)~\cite{silver2017mastering} (upper part of Figure~\ref{fig:intro_comparison}). T-GNNExplainer is a \textit{post-hoc} explanation method because explanations are generated based on a pretrained base model (i.e., TGNN).
Despite its effectiveness, T-GNNExplainer has two major drawbacks that originate from its post-hoc manner of generating explanations.
First, since a temporal graph consistently encounter changes in the topology due to its dynamic nature, the base model needs to be consistently retrained. This results in the repeated retraining of the explanation model (i.e., navigator and explorer) based on the retrained base model (i.e., TGNN), which makes T-GNNExplainer inefficient especially when the base model is large.
The complexity issue aggravates as the explanations are generated based on MCTS, which is a highly inefficient search algorithm.
Second, since post-hoc explanation methods provide explanations by examining the behavior of an already trained base model, it becomes challenging to fully comprehend the learning process of the base model, and to provide accurate explanations~\cite{zhang2022protgnn}.

As a solution to address the drawbacks of post-hoc explanation methods when applied to temporal graphs, we propose to allow the model to simultaneously perform both predictions and explanations in temporal graphs by generating intrinsic explanations within the model itself (i.e., \textit{built-in} explanation
method) (lower part of Figure~\ref{fig:intro_comparison}). Existing post-hoc explanation methods for temporal graphs such as T-GNNExplainer focus on deriving subgraphs that generate predictions as similar as possible to the predictions of the base model for the target event. 
On the other hand, since our proposed built-in explanation method does not have a base model, we induce interactions between the target event representation at the current timestamp and the candidate event representations at past timestamps to extract the importance probability of each candidate event.
To consider the interaction of representations at different timestamps, we generate time-aware representations by taking into account the time spans between the target event and candidate events.
The time-aware representations facilitates our model to identify important candidate events that are used as explanations for the model predictions. 
Our goal is to detect significant past events in temporal graphs based on the constructed time-aware representations.

\begin{figure}[t] 
\begin{center}
\includegraphics[width=0.70\linewidth]{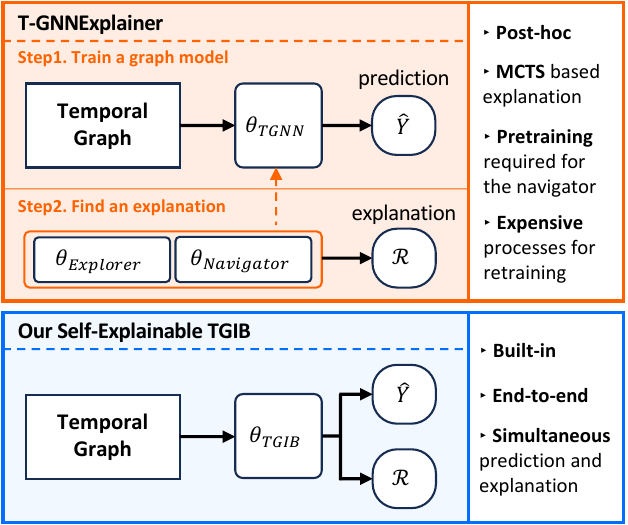}
\end{center}
\vspace{-3mm}
\caption{
Comparison between T-GNNexplainer and TGIB.}
\label{fig:intro_comparison}
\vspace{-5ex}
\end{figure}

\looseness=-1
To this end, we propose a novel \textit{built-in} explanation framework for temporal graphs, called Self-Explainable \textbf{T}emporal Graph Networks based on \textbf{G}raph \textbf{I}nformation \textbf{B}ottleneck (TGIB). 
The main idea is to build an {end-to-end model} that can simultaneously generate predictions for temporal graphs along with explanations based on the Information Bottleneck (IB) approach, which enables the model to detect important subgraphs by leveraging the time-aware representations.
Specifically, TGIB considers the interaction between the target event and candidate events to extract important candidate events, eventually predicting the occurrence of the target event. 
We formalize the prediction process for temporal graphs with the IB to restrict the flow of information from candidate events to predictions by injecting stochasticity into edges~\cite{shannon1948mathematical}. The stochasticity for label-relevant components decreases during training, whereas the stochasticity for label-irrelevant components is maintained. This difference in stochasticity eventually provides explanations for the occurrence of the target events. We expect improved generalization performance of TGIB by penalizing the amount of information from input data.

We conducted extensive experiments to evaluate the prediction performance and the explainability of TGIB in the event occurrence prediction task. Our results show that TGIB outperforms existing link prediction models in both transductive and inductive environments. We also evaluated the explainability of TGIB in capturing the label information by evaluating its prediction performance with only the detected explanation graphs. 
We also demonstrate that explanation graphs with varying sparsities exhibit a higher explanation performance than existing explanation models.
Overall, our results show that TGIB not only significantly improves the performance of the event occurrence prediction, but also provides superior performance in terms of explanations. To the best of our knowledge, this is the first study to simultaneously perform predictions and explanations in temporal graphs. 

In summary, our main contributions are summarized as follows:
\begin{itemize}[leftmargin=0.5cm]
    \item We propose an explainable graph neural network model for temporal graphs that can simultaneously perform prediction and explanation.
    \item We provide a theoretical background of TGIB based on the IB framework regarding the extraction of important past events for predicting the occurrence of the target event.
    \item Extensive experiments demonstrate that TGIB outperforms state-of-the-art methods in terms of both link prediction performance and explainability. 
\end{itemize}
  
\vspace{1mm}

\section{Related work} \label{sec:relatedwork}
\subsection{Temporal Graph Neural Networks} 
Unlike ordinary Graph Neural Networks for a static graph, TGNNs produce dynamic node embeddings from a temporal graph evolving with a series of events.
Early methods~\cite{dai2016deep, kumar2019predicting, trivedi2019dyrep} used RNN-based architectures to produce temporal node embeddings by only considering nodes involved in each of specific events.
As these methods merely use direct connectivity represented as a single event, the Self-Attention Mechanism (SAM)~\cite{vaswani2017attention} was adopted in recent methods for modeling more complex spatial and temporal relationships.
TGAT~\cite{xu2020inductive} applied SAM for simultaneous modeling of both spatial and temporal relationships with functional time encodings based on Bochner's theorem. 
TGN~\cite{rossi2020temporal} first updates the node memory for temporal dependency by using an RNN-based model and computes node embeddings by a SAM-based model with spatial and temporal information similar to~\cite{xu2020inductive}.
TCL~\cite{wang2021tcl} also utilized the SAM-based architecture considering both spatial and temporal dependencies while using a contrastive learning scheme by regarding an interacted node as a positive sample. 
On the other hand, GraphMixer~\cite{cong2023we} showed effectiveness with a simple MLP-based architecture and argued neither RNN nor SAM is mandatory for TGNNs.

\vspace{-1mm}
\subsection{Graph Information Bottleneck}
As the Information Bottleneck (IB) principle enables a model to extract information within the input data that is relevant to the target (i.e., label) information, it has been widely adopted in the various fields of machine learning~\cite{alemi2016deep, peng2018variational, higgins2016beta}.
Inspired by their successes, recent studies regarding the IB principle on graph-structured data were proposed.
GIB~\cite{wu2020graph} extended the IB principle on GNNs by extracting both minimal and sufficient information from both the graph structure and node features.
Extending this, GIB~\cite{yu2020graph} utilized the IB principle for recognizing an important subgraph (i.e., IB-Graph) within the input graph, and then applied the IB-Graph for improving the graph classification performance.
Additionally, VGIB~\cite{yu2022improving} utilized learnable random noise injection in the subgraph recognition process to enable flexible subgraph compression.
Meanwhile, CGIB~\cite{lee2023conditional} applied the graph information bottleneck in molecular relational learning by finding a core subgraph of one molecule based on the paired molecule and the target label.
Furthermore, PGIB~\cite{seo2023interpretable} approached the information bottleneck principle from a prototype perspective to provide prototypes with the subgraph from the input graph which is important for model prediction.

\vspace{-1mm}
\subsection{GNN Explainability} 
\looseness=-1
Although Graph Neural Networks (GNNs) have been shown to be effective on graph data~\cite{kipf2016semi, velivckovic2017graph},  
analyzing the reason of the prediction and the decision-making process of these models has been a long-standing challenge due to their complex architectures.
For this reason, approaches for explainable AI (XAI) have been recently proposed to understand black-box GNNs~\cite{pope2019explainability, ying2019gnnexplainer, vu2020pgm, luo2020parameterized}.
Despite their effectiveness, as these approaches are \textit{post-hoc} methods, they are both ineffective and inefficient when changes occur not only in the training data but also the trained model to be explained. 
Therefore, self-explainable (i.e., \textit{built-in})~\cite{miao2022interpretable, zhang2022protgnn, seo2023interpretable} approaches have been recently gained attention. They contain an explanatory module inside the model to make predictions and explanations simultaneously, addressing the limitations of post-hoc approaches. 
However, all the aforementioned methods are designed for static graphs, and cannot be readily applied to temporal graphs due to their dynamicity.
Recently, T-GNNExplainer~\cite{xia2022explaining} is proposed to give an explanation on TGNNs trained on temporal graphs, which, however, has major drawbacks due to its post-hoc manner as mentioned in Section~\ref{sec:intro}.
Additionally, STExplainer~\cite{tang2023explainable} generates separate explanation graphs for spatial and temporal graphs in traffic and crime prediction. 
However, to explain event occurrence predictions in a way that is easy for humans to understand, it is necessary to generate a comprehensive explanation graph as a set of temporal events.
To this end, we propose TGIB that can simultaneously generate predictions and explanations in temporal graphs by detecting past events that are important for the predictions based on the IB 
principle.

\section{Preliminaries} \label{preliminarysection}

\subsection{Temporal Graph Model}
\looseness=-1
A temporal graph contains a series of continuous events $S= \left\{ e_1, e_2, \cdots \right\} $ with timestamps, where $e_i= \left\{ u_i, v_i, t_i, \textrm{att}_i \right\} $ indicates an interaction event between node $u_i$ and $v_i$ at timestamp $t_i$ with edge attribute $\textrm{att}_i$. The event set $S$ composes a temporal graph $\mathbf{G} =(\mathbf{V}, \mathbf{E})$, where $\mathbf{E}$ is regarded as edges with timestamps and $\mathbf{V}$ represents the nodes included in $\mathbf{E}$. Since the definitions of $\mathbf{V}$ and $\mathbf{E}$ are interdependent, we consider a temporal graph $\mathbf{G}$ as a set of events.  
We define $\mathbf{G}^k$ as the graph constructed immediately before the timestamp $t_k$ which includes all events $\left\{ e_1, e_2, \cdots , e_{k-1} \right\} $ excluding $e_{k}$. Let $f$ denote a self-explainable temporal graph model that we aim to learn. The model $f$ simultaneously predicts the occurrence of interactions between two nodes at a certain timestamp, and explains the reason for its prediction regarding the occurrence or absence of the event $e_k$.
The output of the model $f$, i.e., $f(\mathbf{G}^k)[e_k]$, consists of two components: the label prediction indicating whether the event $e_k$ is present or absent, i.e., $\hat{Y}$, and an explanation for the presence or absence of $e_k$, i.e., $\mathcal{R}^k\in\mathcal{G}^k$, where $\mathcal{G}^k$ is the $L$-hop computation graph of $e_k$, and 
$\mathcal{R}^k$ is a subgraph of $\mathcal{G}^k$ considered as important events. Note that $\mathcal{G}^k$ of $e_k=\left\{ u_k, v_k, t_k, \textrm{att}_k \right\} $ is a combination of the $L$-hop computation graphs of node $u_k$ and node $v_k$.

\subsection{Graph Information Bottleneck}
The mutual information denoted as $I(X; Y)$ between two random variables $X$ and $Y$ is formally defined as:
\begin{equation}\label{eq:MI}
I(X;Y) = \int_{X}^{} \int_{Y}^{} p(x,y)\log\frac{p(x,y)}{p(x)p(y)}\text{d}x \text{d}y.
\end{equation}
Given an input denoted as $X$ and its corresponding label $Y$, the Information Bottleneck (IB)~\cite{tishby2000information} aims to optimize the following objective function in order to derive a bottleneck variable $Z$ as:
\begin{equation}\label{eq:IB}
\min_Z -I(Y;Z)+\beta I(X; Z),
\end{equation}
\noindent where $\beta$ enables the regulation of the balance between two terms as the Lagrange multiplier.
\looseness=-1
The IB principle has found recent application in learning a bottleneck graph, referred to as the IB-Graph, for a given graph $\mathcal{G}$, which aims to preserve the minimal necessary information concerning the properties of $\mathcal{G}$ while capturing the maximal information about the label $Y$.
This approach compressively identifies label-related information from the original graph $\mathcal{G}$, inspired by the Graph Information Bottleneck (GIB) principle by optimizing the objective function as:
\begin{equation}\label{eq:gib}  
\min_{\mathcal{G}_{sub}} - I(Y; \mathcal{G}_{sub}) + \beta I(\mathcal{G} ; \mathcal{G}_{sub}),
\end{equation}
where $\mathcal{G}_{sub}$ is the IB-Graph and $Y$ is the label of $\mathcal{G}$.
The first term aims to maximize the mutual information between the graph label and the subgraph to include as much graph label information $Y$ as possible in the subgraph $\mathcal{G}_{sub}$.
The second term aims to minimize the mutual information between the original graph and the subgraph to include the original graph $\mathcal{G}$ in the subgraph $\mathcal{G}_{sub}$ to a minimum extent.

\section{Methodology} \label{sec:method}
We present our proposed method, called TGIB. We introduce GIB-based objective for a temporal graph (Section~\ref{subsec:obj}), time-aware event representation (Section~\ref{subsec:ekej}),  neural network parameterization for each term of the objective function (Section~\ref{subsec:compression} and \ref{subsec:prediction}) and spurious correlation removal (Section~\ref{subsec:inductivebias}). 
The entire process is presented in Figure~\ref{fig:architecture}.
We also include the pseudocode in Appendix~\ref{apx:alg}.

\subsection{GIB-based Objective for Temporal Graph}  \label{subsec:obj}

We provide the objective of the Graph Information Bottleneck for temporal graphs. 
TGIB extracts a bottleneck code $\mathcal{R}^k$ for the target edge $e_k$ from its $L$-hop neighborhood $\mathcal{G}^k$.
Specifically, the bottleneck code $\mathcal{R}^k$ is a subgraph of $e_k$'s $L$-hop computation graph. The bottleneck mechanism on neighborhood information provides explanations for the prediction of $e_k$. 
The objective function of the graph information bottleneck is provided as follows:
\begin{equation}\label{eq:tgib}  
\min_{\mathcal{R}^k} \underbrace{- I \left(Y_k ; \mathcal{R}^k \right)}_\text{Section~\ref{subsec:prediction}} + \beta \underbrace{I \left(\mathcal{R}^k ; e_k, \mathcal{G}^k \right)}_\text{Section~\ref{subsec:compression}}.
\end{equation}
where $Y_k$ is the label information regarding the occurrence of $e_k$.
The first term $- I \left(Y_k ;\mathcal{R}^k \right)$ allows $\mathcal{R}^k$ to sufficiently learn label-relevant information, while the second term $I \left(\mathcal{R}^k ; e_k, \mathcal{G}^k \right)$ ensures that $\mathcal{R}^k$ efficiently includes only important information related to $e_k$ and $\mathcal{G}^k$ and removes unnecessary information. However, directly optimizing Equation~\ref{eq:tgib} is challenging because of the difficulty of directly calculating the mutual information.

Inspired by \cite{alemi2016deep}, we obtain the upper bound of $-I \left(Y_k ; \mathcal{R}^k \right)$ in Equation~\ref{eq:tgib} as follows:
\begin{equation}\label{eq:prediction}  
\begin{aligned}
\small
- I(Y_k ; \mathcal{R}^k) &= \mathbb{E}_{Y_k, e_k,\mathcal{R}^k}\left[ -\log p \left( Y_k | \mathcal{R}^k \right) \right] -H \left( Y \right)  \\ 
&\leq \mathbb{E}_{Y_k, \mathcal{R}^k}\left[ -\log q_{\theta} \left( Y_k | \mathcal{R}^k \right) \right] - H \left( Y \right)\\
&\leq \mathbb{E}_{Y_k,\mathcal{R}^k}\left[ -\log q_{\theta} \left( Y_k | \mathcal{R}^k \right) \right] \coloneqq \mathcal{L}_{\textrm{cls}},
\end{aligned}
\end{equation}
where $q_{\theta} \left( Y_k | \mathcal{R}^k \right)$ is the variational approximation of $p \left( Y_k |\mathcal{R}^k \right)$. 
We model $q_{\theta} \left( Y_k |\mathcal{R}^k \right)$ as a predictor which outputs predictions for the occurrence of $e_k$ based on the subgraph $\mathcal{R}^k$. 
We can maximize $I(Y_k ;\mathcal{R}^k)$ by minimizing $\mathcal{L}_{\textrm{cls}}$ using the predictor.

Moreover, we obtain the upper bound of $I \left(\mathcal{R}^k ; e_k, \mathcal{G}^k \right)$ in Equation~\ref{eq:tgib} as follows:
\begin{equation}\label{eq:compression}  
\begin{aligned}
\small
I(\mathcal{R}^k ; e_k, \mathcal{G}^k) 
&= \mathbb{E}_{\mathcal{R}^k, e_k, \mathcal{G}^k} \left[ \log p \left( \mathcal{R}^k | e_k, \mathcal{G}^k \right) -  \log p \left( \mathcal{R}^k \right) \right]\\
&\leq \mathbb{E}_{\mathcal{R}^k, e_k, \mathcal{G}^k}\left[ \log p \left( \mathcal{R}^k |e_k, \mathcal{G}^k \right) -  \log q \left( \mathcal{R}^k \right) \right]\\
&\leq \mathbb{E}_{e_k, \mathcal{G}^k} \left[ \textrm{KL} \left[ p \left( \mathcal{R}^k |e_k, \mathcal{G}^k \right) \| q \left( \mathcal{R}^k \right) \right]\right] \coloneqq \mathcal{L}_{\textrm{MI}},
\end{aligned}
\end{equation}
where $q \left( \mathcal{R}^k \right)$ is the variational approximation of $p \left( \mathcal{R}^k \right)$ and  KL represents the Kullback-Leibler(KL) divergence. 
$q \left( \mathcal{R}^k \right)$ can be flexibly applied to various distributions, including normal distributions. We can minimize the upper bound of $I(\mathcal{R}^k ; e_k, \mathcal{G}^k)$ by minimizing $\mathcal{L}_{\textrm{MI}}$.

Finally, we obtain the final loss function for the graph information bottleneck as follows: 
$\mathcal{L}_{total} = \mathcal{L}_{\textrm{cls}} + \beta \mathcal{L}_{\textrm{MI}}$. 

\subsection{Time-aware event representation} \label{subsec:ekej}
We generate time-aware event representations to capture the temporal information of events. Inspired by \cite{xu2020inductive}, we obtain node representations by performing self-attention based on temporal encoding.
We consider the neighboring nodes for node $z$ at time $t$ as $\mathcal{N}(z; t) = \left\{ z_{1}, z_{2},  \cdots ,  z_{n} \right\}$, where $n$ is the number of neighbors.
The interaction between $z$ and one of its neighbor $z_{i}$ has edge attribute $att_{z,i}\in\mathbb{R}^{f_\text{edge}}$ and occurs at time $t_{z,i}$, which is earlier than $t$.
We use the representations of neighbors, attributes of their interactions, and their temporal information as the input to the self-attention layer. 
We use $h_{z}^{(l)}(t)\in\mathbb{R}^d$ to denote the representation of node $z$ at time $t$ in the $l$-th layer, where $h_{z}^{(0)}(t)$ is the raw node feature of node $z$, denoted as  $x_{z}\in\mathbb{R}^{f_\text{node}}$, that is invariant over time $t$.
For the self-attention mechanism, we define the query, key and value as:
\begin{small}
\begin{align}\label{eq:qkv} 
Q^{(l)}(t) &= \left[ \; h_{z}^{(l-1)}(t)\|att_{z,0} \|\Phi_{d_T}(0) \; \right],  \nonumber \\ 
\vspace{10pt}
K^{(l)}(t)
&= \begin{bmatrix}
K^{(l)}_1(t)\\
\vdots  \\
K^{(l)}_n(t) \\ 
\end{bmatrix}
= \begin{bmatrix}
h_{z_1}^{(l-1)}(t_{z,1})\|att_{z,1} \|\Phi_{d_T}(t-t_{z,1}) \\
\vdots  \\
h_{z_n}^{(l-1)}(t_{z,n})\|att_{z,n} \|\Phi_{d_T}(t-t_{z,n})\\ 
\end{bmatrix}, \nonumber\\
V^{(l)}(t)
&= \begin{bmatrix}
V^{(l)}_1(t)\\
\vdots  \\
V^{(l)}_n(t) \\ 
\end{bmatrix}
= \begin{bmatrix}
h_{z_1}^{(l-1)}(t_{z,1})\|att_{z,1} \|\Phi_{d_T}(t-t_{z,1}) \\
\vdots  \\
h_{z_n}^{(l-1)}(t_{z,n})\|att_{z,n} \|\Phi_{d_T}(t-t_{z,n})\\ 
\end{bmatrix},
\end{align}
\end{small}

\noindent where $\Phi_{d_T}:\mathbb{R} \rightarrow \mathbb{R}^{d_T}$ is time encoding that provides a continuous functional mapping from the time domain to the vector space with $d_T$ dimensions. 
According to the translation-invariant assumption of time encoding, we use $\left\{ t-t_{z,1}, t-t_{z,2}, \cdots, t-t_{z,n} \right\}$ as the interaction times. 
This means that we only consider the time span since time encoding is designed to measure the temporal distance between nodes, which is more important than the absolute value of time.

\begin{figure*}[t]
  \centering
  \includegraphics[width=0.99\textwidth]{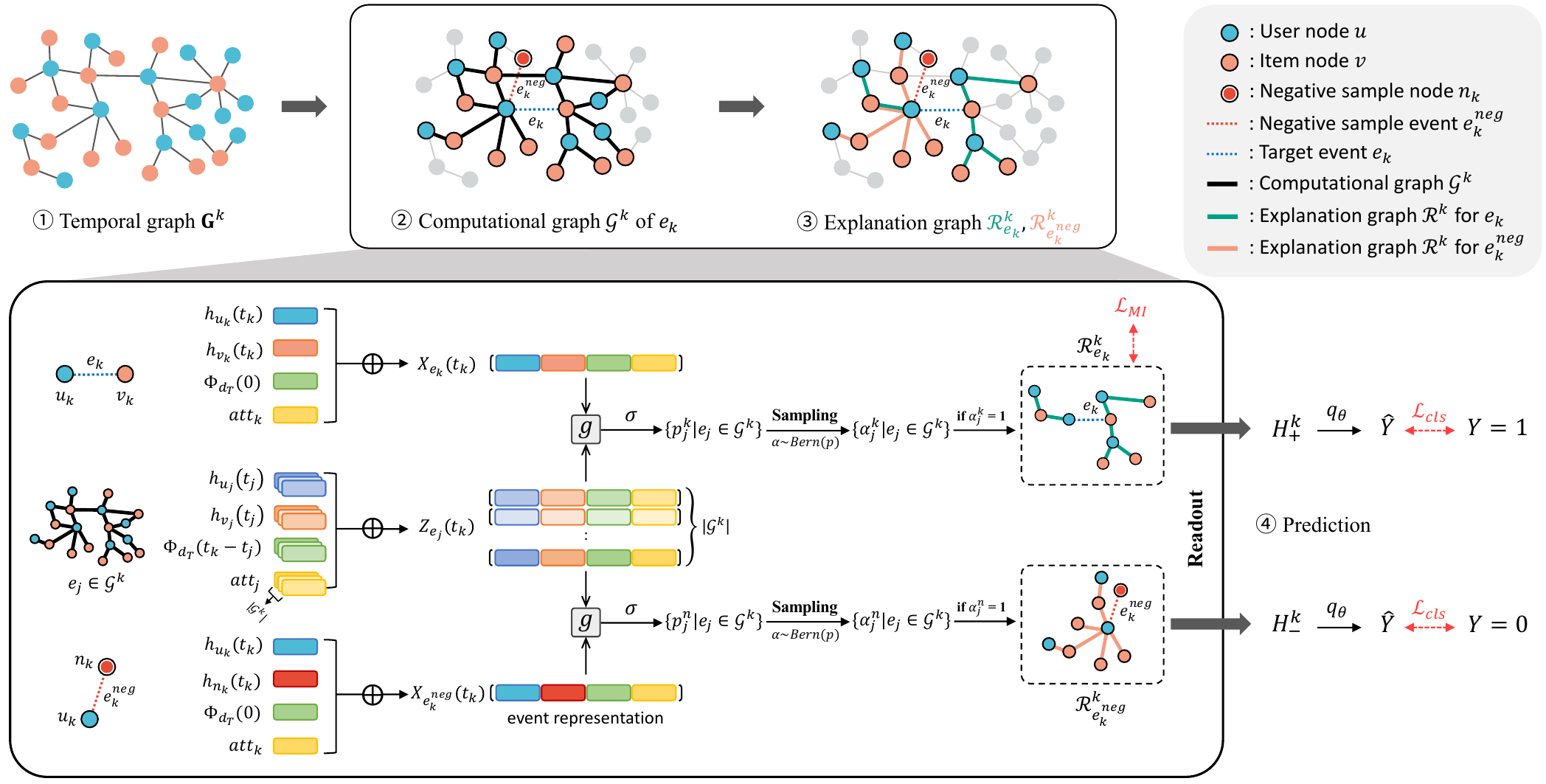}
  \vspace{-2mm}
  \caption{The architecture of our proposed TGIB.}  
  \label{fig:architecture}
\end{figure*}

Each node collects information from its neighboring nodes, and in this process, the attention weights ${\left\{\alpha^{(l)}_{i}(t)\right\}}_{i=1}^n $ are defined as:
\begin{equation}\label{eq:attweight}  
\small
\alpha^{(l)}_i (t) = \text{softmax}\left(\frac{ Q^{(l)}(t) K^{(l)}_i(t)^T }{\sqrt{(d+f_\text{edge}+d_T)}}\right)\in\mathbb{R}^n,
\end{equation}
where $K^{(l)}_i(t)$ is the $i$-th row of $K^{(l)}(t)$. The attention weight $\alpha_{i}^{(l)}(t)$ indicates the contribution of node $z_i$ to the features of node $z$ within the topological structure $\mathcal{N} \left( z; t \right)$, considering the interaction time $t$ with $z$ computed at layer $l$. 
We utilize self-attention to incorporate temporal interactions into node features and structural information.
The representation for a node $z_i$ within the neighborhood $\mathcal{N} \left( z; t \right)$ is calculated as $\alpha_i^{(l)}(t)\cdot V^{(l)}_i(t)$ based on the attention weight $\alpha^{(l)}_i(t)$.
Finally, we obtain the hidden neighborhood representations as follows\footnote{To remove clutter, we omit weights for query, key, and value, i.e., $W^Q, W^K$ $W^V$, that exist for each layer $l$.}:
\begin{equation}\label{eq:attn}
\small
\begin{split}
\tilde{h}_{z}^{(l)}(t) &= \textrm{Attn}\left(Q^{(l)}(t), K^{(l)}(t), V^{(l)}(t) \right) \\& =  \textrm{softmax}\left( \frac{Q^{(l)}(t) {K^{(l)}(t)}^T} {\sqrt{(d+f_\text{edge}+d_T)}}\right) V^{(l)}(t).
\end{split}
\end{equation}
We concatenate the neighborhood representation $\tilde{h}_{z}^{(l)}(t)$ with the feature of node $z$, i.e., $x_{z}$, and use it as the input to a feed-forward network as follows:
\begin{align}\label{eq:attn}  
\nonumber h_{z}^{(l)}(t) &= \textrm{FFN}\left(\left[\tilde{h}_{z}^{(l)}(t) \| x_z \right]\right)  \\
&= \textrm{ReLU} \left( \left[ \tilde{h}_{z}^{(l)}(t)\| x_z \right]W_0^{(l)}+b_0^{(l)}\right) W_1^{(l)}+b_1^{(l)},
\end{align}
where  $h_{z}^{(l)}(t)$ is the time-aware node embedding for node $z$ at time $t$ as the output.

Finally, we construct the time-aware event representation for the target event $e_k=\left\{ u_k, v_k, t_k, \textrm{att}_k \right\} $ with the representation of the two nodes $u_k$ and $v_k$, time encoding and edge attribute $\textrm{att}_k$ as follows:
\begin{equation}\label{eq:e_k}  
X_{e_{k}}(t_k) = \left[  \; h_{u_{k}}(t_k) \; \| \; h_{v_{k}}(t_k) \; \|  \; \Phi_{d_{T}}(0)  \; \|  \; \textrm{att}_k  \; \right],
\end{equation}
where $X_{e_{k}}(t_k)$ denotes the time-aware representation for the target event $e_k$ at time $t_k$.
Additionally, we define the candidate event representation, denoted as $Z_{e_j}$, which may potentially be included within the explanation graph $\mathcal{R}^k$ as follows: 
\begin{equation}\label{eq:e_j}  
Z_{e_{j}}(t_k) = \left[  \; h_{u_{j}}(t_j)  \; \|  \; h_{v_{j}}(t_j)  \; \|  \; \Phi_{d_{T}}(t_k-t_j)  \; \|  \; \textrm{att}_j \right],
\end{equation}
where all $e_j$ satisfy the following condition : $e_j \in \mathcal{G}^k$, where $1 \leq j \leq k-1$.
In the time encoding $\Phi_{d_{T}}$ of the above equations, similar to node representation, we consider the time span to measure the temporal distance between each interaction.
It allows time-aware explanations to sufficiently incorporate temporal distance information and ensures that the importance scores of multiple events occurring between the same pair of nodes over time are dependent on the time span.
Therefore, we regard the interaction time as $t_k-t_j$, where $t_k$ is the occurrence time of the target event $e_k$, and $t_j$ is the occurrence time of the candidate event $e_j$. 

\subsection{Minimizing $I(\mathcal{R}^k ; e_k,  \mathcal{G}^k)$ in Equation~\ref{eq:compression}} \label{subsec:compression}

In this section, we propose a parameterized time-aware bottleneck to model $p(\mathcal{R}^k |e_k,  \mathcal{G}^k)$ given in Equation~\ref{eq:compression}. To alleviate computational difficulty, we decompose $(\mathcal{R}^k |e_k,  \mathcal{G}^k)$ into a multivariate Bernoulli distribution as follows:
\begin{equation}\label{eq:p(Rk|ekGk)}  
p(\mathcal{R}^k |e_k,  \mathcal{G}^k) = \prod_{e_j \in \mathcal{R}_k} p_j^k 
\ \ \cdot \prod_{e_j \in \mathcal{G}^k \textrm{\textbackslash} \mathcal{R}_k}(1-p_j^k), 
\end{equation}
where $p_j^k$ is the probability of $e_j$ given $e_k$ and $\mathcal{G}^k$, i.e., $p(e_j|e_k,  \mathcal{G}^k)$, following Bernoulli distribution.
We apply the Gumbel-Softmax technique ~\cite{jang2016categorical} for sampling in order to allow gradients to propagate from the classifier to the time-aware bottleneck module in optimization. 
Each $p_j^k$ is computed as the output of an MLP that takes the target event representation $X_{e_k}$ and the candidate event representation $Z_{e_j}$ as input as follows:
\begin{equation}\label{eq:pjk}  
p_j^k = p(e_j|e_k,  \mathcal{G}^k) = \sigma \left( g \left(X_{e_{k}}(t_k), \;  Z_{e_{j}}(t_k)  \right)\right)
\end{equation}
where $\sigma \left( \cdot \right)$ is the sigmoid function and $g$ is an MLP. 
A large $p_j^k$ indicates that $e_j$ is important for predicting $e_k$ in $\mathcal{G}^k$.

Next, we define the variational approximation $q(\mathcal{R}^k)$ of the marginal distribution $p(\mathcal{R}^k)$. 
For every edge $e$ in the graph $\mathcal{G}^k$, we sample $\alpha_e' \sim Bern(r)$, where $r\in \left[ 0, 1 \right]$ is a predefined hyperparameter. 
We eliminate all edges in $\mathcal{G}^k$ and reinstate all edges with $\alpha_e'=1$.
We assume that the graph obtained in this process is $\mathcal{R}^k$.
Consequently, we use a multivariate Bernoulli distribution for $q(\mathcal{R}^k)$ as follows : 
\begin{equation}\label{eq:q(Rk)}  
q(\mathcal{R}^{k})=r^{|\mathcal{R}^{k}|} (1-r)^{|\mathcal{G}^{k}|-|\mathcal{R}^{k}|}.
\end{equation}

Finally, the mutual information loss $\mathcal{L}_{\textrm{MI}}$ for the time-aware bottleneck in Equation~\ref{eq:compression} is calculated as follows:

\begin{equation}\label{eq:Lmi}  
\begin{aligned}
\mathcal{L}_{\textrm{MI}} &= \mathbb{E}_{ \mathcal{G}^k} \left[ KL \left[ p(\mathcal{R}^k | e_k, \mathcal{G}^k) \| q(\mathcal{R}^k) \right]\right] \\
&= \mathbb{E}_{p(e_k, \mathcal{G}^k)}\left[ \sum_{e_j\in\mathcal{G}^k} p_j^k \log \frac{p_j^k}{r} + (1-p_j^k) \log \frac{1-p_j^k}{1-r} \right].
\end{aligned}
\end{equation}

\subsection{Minimizing $-I(Y_k ; \mathcal{R}^k)$ in Equation~\ref{eq:prediction}} \label{subsec:prediction}
The predictor $q_{\theta} \left( Y_k | \mathcal{R}^k \right)$ in Equation~\ref{eq:prediction} provides predictions $Y_k$ based on the bottleneck code $\mathcal{R}^k$. 
We sample stochastic weights from the Bernoulli distribution and obtain a valid event representation $\tilde{Z}_{e_j}(t_k)$ from the sampled $\alpha_j^k$ as follows:
\begin{equation}\label{eq:Ytilda}  
\tilde{Z}_{e_j}(t_k) = \alpha_j^k Z_{e_j}(t_k), \quad \alpha_j^k \sim Ber(p_j^k),
\end{equation}
where $Z_{e_j}(t_k)$ is the candidate event representation obtained from Equation~\ref{eq:e_j}.

To ensure that gradients can be computed with respect to $p_j^k$, we use the Gumbel-Softmax reparameterization trick ~\cite{jang2016categorical}.
We obtain the representation of $\mathcal{R}^k$, denoted as ${H}^k_{+}$, extracted from each $\tilde{Z}_{e_j}(t_k)$ as follows:
\begin{equation}\label{eq:Rkembedding} 
{H}_{+}^k = \textrm{Readout} \left[  \left\{ \tilde{Z}_{e_j}(t_k) | e_j \in \mathcal{G}^k \right\} \right].
\end{equation}



\noindent \textbf{Negative Sample.} 
We utilize negative samples to effectively train the predictor $q_{\theta} \left( Y_k | \mathcal{R}^k \right)$.
To generate a negative sample of $e_k=\left\{ u_k, v_k, t_k, \textrm{att}_k \right\} $, we fix the node $u_k$ and replace $v_k$ by randomly sampling a node from the entire graph $\mathbf{G}$.
We denote the randomly sampled node as $n_k$, and the event representation $X_{e_k^\text{neg}}(t_k)$ corresponding to the negative sample $n_k$ is defined as follows:
\begin{equation}\label{eq:e_n}  
X_{e_k^\text{neg}}(t_k) = \left[  \; h_{u_{k}}(t_k) \; \| \; h_{n_{k}}(t_k) \; \|  \; \Phi_{d_{T}}(0)  \; \|  \; att_k  \; \right] 
\end{equation}


In the same manner as $X_{e_k}(t_k)$, we extract the bottleneck code based on $X_{e_k^\text{neg}}(t_k)$ along with a candidate event $Z_{e_j}(t_k)$, and sample stochastic weights from the Bernoulli distribution. 
Consequently, using the Bernoulli variables derived based on $X_{e_k^\text{neg}}(t_k)$, we select important events from candidate events and obtain the selected candidate graph embedding ${H}^k_{-}$ through a readout function.

Finally, we use the time-aware link prediction loss function as follows:
\begin{equation}\label{eq:Lcls} 
\small
\mathcal{L}_{\textrm{cls}} = \sum_{
e_k \in S
} -\log \left[ \sigma\left( q_{\theta}( X_{e_k}, {H}^k_{+} ) \right)\right]
 - N\cdot\mathbb{E}_{\text{neg} \sim P_n} \log \left[ \sigma\left( 
  q_{\theta} ( X_{e_k^\text{neg}}, {H}^k_{-} ) \right) \right], 
\end{equation}
where $q_{\theta}$ is an MLP,
$N$ is the number of negative samples, and $P_n$ is a distribution from which negative samples are sampled.
\vspace{2mm}

\noindent \textbf{Explanability.} The explainability of TGIB is established by injecting stochasticity into past candidate events. $\mathcal{L}_{\textrm{MI}}$ in Equation ~\ref{eq:Lmi} aims to assign high stochasticity to all candidate events, while $\mathcal{L}_{\textrm{cls}}$ in Equation ~\ref{eq:Lcls} simultaneously learns to reduce the stochasticity for explanation graphs that are important for the occurrence of $e_k$. 
We generate an importance score $p_j^k$ derived from the interactions of events at different timestamps, i.e., $e_k$ and $e_j$, which allows the bottleneck code to help generate a time-aware explanation.
TGIB can rank all candidate events according to $p_j^k$ and detect the top-ranked candidate events as explanation graphs.

\subsection{Spurious Correlation Removal}
\label{subsec:inductivebias}
\begin{figure}[H] 
\begin{center}
\includegraphics[width=0.85\linewidth]{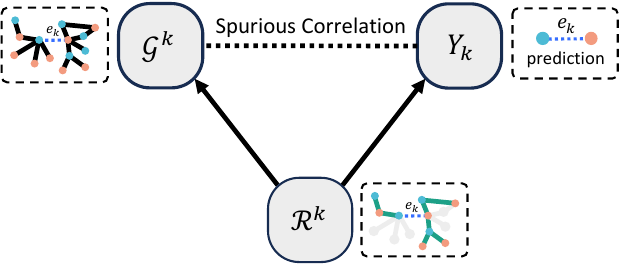}
\end{center}
\vspace{-3mm}
\caption{
Spurious Correlation Removal of TGIB.}
\label{fig:inductivebias}
\vspace{-1.8ex}
\end{figure}

TGIB eliminates spurious correlations within the input data and ensures interpretability. If there is a correspondence between the explanation of event occurrence $\mathcal{R}^{k*}$  and the label Y, we can prove that $\mathcal{R}^{k*}$  is the optimal solution for the objective function of TGIB (i.e., Equation~\ref{eq:tgib}).

\vspace{1mm} 
\begin{prop} \label{prop:tgib}
Assume that each $\mathcal{G}^k$ contains $\mathcal{R}^{k*}$, which determines $Y_k$. In other words, for some deterministic invertable function $f$ with randomness $\epsilon$ that is independent of $\mathcal{G}^k$, it satisfies $Y = f(\mathcal{R}^{k*}) + \epsilon$. 
Then,  for any $\beta \in [0,1]$, the optimal $\mathcal{R}^{k}$ that minimizes $-I(Y_k;\mathcal{R}^k)+\beta I(\mathcal{R}^k;e_k,\mathcal{G}^k)$ is $\mathcal{R}^{k*}$.
\end{prop}

\begin{proof}
We can obtain the following derivation:

\begin{footnotesize}
\begin{align*}\label{eq:Lc} 
& \quad \min_{\mathcal{R}_k} {-I(Y_k;\mathcal{R}^k) +\beta I(\mathcal{R}^k;e_k,\mathcal{G}^k)} \\
&=\min_{\mathcal{R}_k} {-I(Y_k;\mathcal{R}^k) - \beta I(Y_k;e_k,\mathcal{G}^k |  \mathcal{R}^k) + \beta I(Y_k,\mathcal{R}^k ; e_k,\mathcal{G}^k)}\\
&=\min_{\mathcal{R}_k} {-I(Y_k;\mathcal{R}^k) - \beta I(Y_k;e_k|\mathcal{G}^k, \mathcal{R}^k) - \beta I(Y_k; \mathcal{G}^k|\mathcal{R}^k) + \beta I(Y_k,\mathcal{R}^k ; e_k,\mathcal{G}^k)}\\
&=\min_{\mathcal{R}_k} {-I(Y_k;\mathcal{G}^k,\mathcal{R}^k) +I(Y_k ;\mathcal{G}^k|\mathcal{R}^k) - \beta I(Y_k; \mathcal{G}^k|\mathcal{R}^k) + \beta I(Y_k,\mathcal{R}^k ; e_k,\mathcal{G}^k)}\\
&=\min_{\mathcal{R}_k} {I(Y_k;\mathcal{G}^k|\mathcal{R}^k) - \beta I(Y_k; \mathcal{G}^k|\mathcal{R}^k) + \beta I(\mathcal{R}^k;e_k,\mathcal{G}^k|Y_k)+\beta I(Y_k ; e_k,\mathcal{G}^k)}\\
&=\min_{\mathcal{R}_k} { (1-\beta)I(Y_k ; \mathcal{G}^k|\mathcal{R}^k) + \beta I(\mathcal{R}^k;e_k,\mathcal{G}^k|Y_k)},
\end{align*}
\end{footnotesize}
where, since $\mathcal{R}^k$ is a subgraph of $\mathcal{G}^k$, implying $(\mathcal{G}^k, \mathcal{R}^k)$ holds no additional information over  $\mathcal{G}^k$, it follows $-I(Y_k; e_k|\mathcal{G}^k,\mathcal{R}^k)=-I(Y_k; e_k|\mathcal{G}^k)$ in the third equation and $-I(Y_k;\mathcal{G}^k,\mathcal{R}^k)=-I(Y_k;\mathcal{G}^k)$ in the fourth equation. 
The derivation of the equations is based on the chain rule of mutual information, and since the objective is to find the $\mathcal{R}^k$ that optimizes the equation, the terms that are irrelevant to $\mathcal{R}^k$ have been removed (detailed derivation is provided in Appendix~\ref{apx:prop1}).

The above derivation shows that $\mathcal{R}^k$ that minimizes $(1-\beta)$ $I(Y_k ; \mathcal{G}^k|\mathcal{R}^k) + \beta I(\mathcal{R}^k;e_k,\mathcal{G}^k|Y_k)$ can also minimize $-I(Y_k;\mathcal{R}^k)+\beta I(\mathcal{R}^k;e_k,\mathcal{G}^k)$.
Since mutual information is always non-negative, $(1-\beta)I(Y_k ; \mathcal{G}^k|\mathcal{R}^k) + \beta I(\mathcal{R}^k;e_k,\mathcal{G}^k|Y_k)$ reaches its minimum when both $I(Y_k ; \mathcal{G}^k|\mathcal{R}^k)$ and $I(\mathcal{R}^k;e_k,\mathcal{G}^k|Y_k)$ are equal to $0$.

$Y = f(\mathcal{R}^{k*}) + \epsilon$, where $\epsilon$ is independent from $\mathcal{G}^k$, implies that  $I(Y_k ; \mathcal{G}^k|\mathcal{R}^{k*}) = 0 $. Similarly, $\mathcal{R}^{k*} = f^{-1}(Y-\epsilon)$ where $\epsilon$ is independent from $\mathcal{G}^k$ implies that  $I(\mathcal{R}^{k*};e_k,\mathcal{G}^k|Y_k) = 0$. Therefore, the following holds: $(1-\beta)I(Y_k ; \mathcal{G}^k|\mathcal{R}^{k*}) + \beta I(\mathcal{R}^{k*};e_k,\mathcal{G}^k|Y_k)=0$. In other words, the optimal $\mathcal{R}^{k}$ that minimizes $-I(Y_k;\mathcal{R}^k)+\beta I(\mathcal{R}^k;e_k,\mathcal{G}^k)$ is $\mathcal{R}^{k*}$.
\end{proof}

As shown in the Figure~\ref{fig:inductivebias}, $\mathcal{G}^k$ and $Y$ may have a spurious correlation due to the factors excluding $\mathcal{R}^{k*}$ from $\mathcal{G}^k$.
In other words, the correlation between $\mathcal{G}^k \setminus \mathcal{R}^{k*}$ and the label $Y$ is a spurious correlation and not a real factor in determining the label.
Predicting $Y$ based on $\mathcal{G}^k$ has a risk of capturing spurious correlations, which can lead to a decrease in the model's generalization performance. 
According to Proposition~\ref{prop:tgib}, this problem can be solved by optimizing the objective function of TGIB.
It shows that event occurrence can be predicted based on $R^{k*}$ without spurious correlations.
As a result, TGIB can improve the generalization performance of predictions by removing spurious correlations for event occurrence.

\section{Experiments} \label{sec:experiments}
 
In this section, we provide experimental setups (Sec~\ref{sec:experimental_setup}), performances of link prediction (Sec~\ref{subsec:pred_performance}), explanation performances  (Sec~\ref{subsec:expl_performance}), explanation visualization (Sec~\ref{subsec:visualization}) and ablation study of TGIB (Sec~\ref{subsec:ablation}).

\subsection{Experimental Setup} \label{sec:experimental_setup}

\subsubsection*{\textbf{Datasets}}
We use six real-world temporal graph datasets to measure the performance of TGIB. The six datasets include various applications and domains, such as social networks and communication networks. The following describes the details of the datasets used for evaluation.
Data statistics are provided in the Table~\ref{tab:datastat}.

\begin{table}[H]
\centering
\caption{Statistics of datasets used for experiments.}
\vspace{-3mm}
\label{tab:datastat}
\resizebox{\columnwidth}{!}{
\begin{tabular}{c|ccccc}
\toprule[1pt]
Dataset   & Domain   & \#Nodes & \#Edges & \#Edge Features & Duration \\ \hline
Wikipedia & Social   & 9,227   & 157,474 & 172             & 1 month  \\
UCI       & Social   & 1,899   & 58,835  & -               & 196 days \\
USLegis   & Politics & 225     & 60,396  & 1               & 12 terms \\
CanParl   & Politics & 734     & 74,478  & 1               & 14 years \\
Enron     & Social   & 184     & 125,235 & -               & 3 years  \\
Reddit    & Social   & 10,984  & 672,447 & 172             & 1 month  \\ \bottomrule[1pt]
\end{tabular}
}
\end{table}

\begin{itemize}[leftmargin=2mm]
    \item \textbf{Wikipedia} \cite{kumar2019predicting} consists of 9,227 nodes representing editors and Wiki pages, and 157,474 edges representing timestamped post requests. Each edge has a Linguistic Inquiry and Word Count (LIWC) feature vector of the requested text~\cite{pennebaker2001linguistic} with each vector having a length of 172.
    \item \textbf{UCI} \cite{panzarasa2009patterns} is a social network within the online community of University of California, Irvine students spanning 196 days. Nodes represent students, and edges denote messages exchanged between two students, each with a timestamp in seconds.
    \item \textbf{USLegis} \cite{huang2020laplacian} is a co-sponsorship network among U.S. senators. Each node represents legislators, and two legislators are connected if they have jointly sponsored a bill. The weight assigned to each edge represents the number of times two legislators have jointly sponsored bills over a period of 12 terms.
    \item \textbf{CanParl}  \cite{huang2020laplacian} is a network capturing interactions among Canadian Members of Parliament from 2006 to 2019. Nodes represent Members of Parliament, and two members are connected if they both vote in favor of a specific bill. The weight assigned to each edge represents the number of times one member has voted in favor of another member within one year.
    \item \textbf{Enron} \cite{shettyenron} is an email network that includes communication exchanged over a period of three years within the Enron energy company. Nodes represent employees of the Enron company, and edges represent the exchanged emails between two employees.
    \item \textbf{Reddit} \cite{kumar2019predicting} represents a network of posts created by users in subreddits for one month. The nodes represent users and posts, and the edges represent timestamped post requests. Similar to Wikipedia, the features of the edges have LIWC feature vectors of the requested text~\cite{pennebaker2001linguistic}, with each vector having a length of 172.
    
\end{itemize}


\subsubsection*{\textbf{Evalutation Protocol}}
We split the total time $\left[ 0, T \right] $ into $\left[ 0, T_{\textrm{train}} \right]$, $\left[T_{\textrm{train}}, T_{\textrm{val}}\right]$, and $\left[T_{\textrm{val}}, T_{\textrm{test}}\right]$ and use the events occurring within each interval as the training set, validation set, and test set, respectively.
We set $T_{\textrm{train}}=0.7T$ and $T_{\textrm{val}}=0.85T$.
 We set the number of layers to 2 and the dropout rate as 0.1 for the aggregation process within the attention mechanism to obtain time-aware event representations.
 Our model is trained for $10$ epochs using the Adam SGD optimizer with a learning rate of 0.00001.
We set the dimensions of the node embeddings and time encodings to be identical to the raw features of the events.
For the $L$-hop computational graph, we set $L$ as 2.
We evaluate the performance, which is averaged over $5$ independent runs with different random seeds.

\subsection{Link Prediction}  \label{subsec:pred_performance}
\subsubsection*{\textbf{Baselines}}
We use seven TGNN methods, i.e., \textbf{Jodie} \cite{kumar2019predicting}, \textbf{DyRep} \cite{trivedi2019dyrep}, \textbf{TGAT} \cite{xu2020inductive}, \textbf{TGN}  \cite{rossi2020temporal}, \textbf{TCL} \cite{wang2021tcl}, \textbf{CAW-N} \cite{wang2021inductive}, \textbf{GraphMixer} \cite{cong2023we} as baselines.
Details on compared baselines can be found in Appendix~\ref{apx:baseline}.

\subsubsection*{\textbf{Setup}}
To measure performance in various settings, we assess link prediction performance in both transductive and inductive settings.

\begin{itemize}[leftmargin=4mm]
\item \textbf{Transductive Setting.} We use all events occurring in the three intervals (i.e., $\left[ 0, T_{\textrm{train}} \right]$, $\left[T_{\textrm{train}}, T_{\textrm{val}}\right]$, and $\left[T_{\textrm{val}}, T_{\textrm{test}}\right]$) as the training set, validation set, and test set, respectively. This means that during the training time, all events occurring before $T_{\textrm{train}}$ can be observed.  
\item \textbf{Inductive Setting.}  We predict the occurrence of events involving nodes not observed during the training time. Specifically, \textbf{1)} we split the training set, validation set, and test set as shown in Section~\ref{sec:experimental_setup}. \textbf{2)} We randomly select $10\%$ of nodes from the training set, and remove all events containing these nodes from the training set. \textbf{3)} We include the events involving these nodes only in the validation and test sets.  
\end{itemize}

\begin{table*}[t]
\centering
\caption{AP on link prediction in a transductive setting for TGIB and 8 baseline methods over 6 datasets.}
\label{table:transductive}
\setlength{\tabcolsep}{5pt}
\vspace{-3mm}
\resizebox{0.75\textwidth}{!}{
\begin{tabular}{llcccccc}
\cmidrule[1pt]{2-8}
\multirow{10}{*}{\rotatebox[origin=c]{90}{\bf Transductive}} & \bf Model      & \bf Wikipedia   & \bf UCI         & \bf USLegis    &\bf  CanParl    & \bf Enron       & \bf Reddit      \\ \cline{2-8} 
                               & Jodie      & 94.62 ± 0.50   & 86.73 ± 1.00   & 73.31 ± 0.40  & 69.26 ± 0.31 & 77.31 ± 4.20   & 97.11 ± 0.30   \\
                               & DyRep      & 92.43 ± 0.37  & 53.67 ± 2.10 & 57.28 ± 0.71 & 54.02 ± 0.76 & 74.55 ± 3.95  & 96.09 ± 0.11  \\
                               & TGAT       & 95.34 ± 0.10   & 73.01 ± 0.60 & 68.89 ± 1.30  & 70.73 ± 0.72 & 68.02 ± 0.10 & 98.12 ± 0.20 \\
                               & TGN        & 97.58 ± 0.20 & 80.40 ± 1.40 & \underline{75.13 ± 1.30}  & 70.88 ± 2.34 & 79.91 ± 1.30 & \underline{98.30 ± 0.20} \\
                               & TCL        & 96.47 ± 0.16  & 89.57 ± 1.63  & 69.59 ± 0.48 & 68.67 ± 2.67 & 79.70 ± 0.71  & 97.53 ± 0.02  \\
                               & CAW-N        & \underline{98.28 ± 0.20} & 90.03 ± 0.40 & 69.94 ± 0.40  & 69.82 ± 2.34 & \textbf{89.56 ± 0.09}  & 97.95 ± 0.20 \\
                               & GraphMixer & 97.25 ± 0.03  & \underline{93.25 ± 0.57}  & 70.74 ± 1.02 & \underline{77.04 ± 0.46} & 82.25 ± 0.16  & 97.31 ± 0.01  \\ \cmidrule{2-8} 
                               & \cellcolor{grey}TGIB       & \cellcolor{grey}\bf 99.37 ± 0.09       &   \cellcolor{grey}\textbf{93.60 ± 0.24}     & \cellcolor{grey}\bf 91.61 ± 0.34     & \cellcolor{grey}\bf 87.07 ± 0.44      &  \cellcolor{grey}\underline{82.42 ± 0.11}     & \cellcolor{grey}\bf 99.68 ± 0.15       \\ \cmidrule[1pt]{2-8} 
 \end{tabular}}
\end{table*}
\begin{table*}[]
\centering
\caption{AP on link prediction in an inductive setting for TGIB and 8 baseline methods over 6 datasets.}
\label{table:inductive}
\setlength{\tabcolsep}{5pt}
\vspace{-4mm}
\resizebox{0.75\textwidth}{!}{
\begin{tabular}{llcccccc}
\cmidrule[1pt]{2-8}
\multirow{10}{*}{\rotatebox[origin=c]{90}{\bf Inductive}} & \bf Model      & \bf Wikipedia   & \bf UCI         & \bf USLegis    &\bf  CanParl    & \bf Enron       & \bf Reddit      \\ \cline{2-8} 

                               & Jodie      & 93.11 ± 0.40 & 71.23 ± 0.80   & 52.16 ± 0.50 & 53.92 ± 0.94 & 76.48 ± 3.50   & 94.36 ± 1.10   \\
                               & DyRep      & 92.05 ± 0.30  & 50.43 ± 1.20 & 56.26 ± 2.00 & 54.02 ± 0.76 & 66.97 ± 3.80  & 95.68 ± 0.20  \\
                               & TGAT       &  93.82 ± 0.30   & 66.89 ± 0.40 & 52.31 ± 1.50 & 55.18 ± 0.79 & 63.70 ± 0.20& 96.42 ± 0.30 \\
                               & TGN        & 97.05 ± 0.20 & 74.70 ± 0.90 & \underline{58.63 ± 0.37}  & 54.10 ± 0.93 & 77.94 ± 1.02 & 96.87 ± 0.20 \\
                               & TCL        & 96.22 ± 0.17  & 87.36 ± 2.03 & 52.59 ± 0.97 & 54.30 ± 0.66 & 76.14 ± 0.79  & 94.09 ± 0.07  \\
                               & CAW-N        & \underline{97.70 ± 0.20} & 89.65 ± 0.40 & 53.11 ± 0.40  & 55.80 ± 0.69 & \textbf{86.35 ± 0.51}  & \underline{97.37 ± 0.30} \\
                               & GraphMixer & 96.65 ± 0.02  & \underline{91.19 ± 0.42}  & 50.71 ± 0.76 & \underline{55.91 ± 0.82} & 75.88 ± 0.48  & 95.26 ± 0.02  \\ \cmidrule{2-8} 
                               & \cellcolor{grey}TGIB       & \cellcolor{grey}\bf 99.28 ± 0.11      &   \cellcolor{grey}\textbf{91.26 ± 0.16}    & \cellcolor{grey}\bf 86.42 ± 0.16     & \cellcolor{grey}\bf 79.56 ± 0.79     &     \cellcolor{grey}\underline{80.64 ± 0.59}   & \cellcolor{grey}\bf 99.54 ± 0.02       \\ \cmidrule[1pt]{2-8} 
\end{tabular}}
\vspace{-3mm}
\end{table*}

\subsubsection*{\textbf{Experiment Results}}

The experimental results for link prediction in the transductive setting and inductive setting are presented in Table~\ref{table:transductive} and Table~\ref{table:inductive}, respectively.
We measured the mean and standard deviation of Average Precision (AP) on the test set.
We obtained the following observations: \textbf{1)} TGIB demonstrated the highest performance on 5 out of 6 datasets compared to the baselines for the temporal graphs in both transductive and inductive settings.
The remaining dataset (i.e., Enron) showed the second-highest performance compared to the baselines. We attribute the superior performance of TGIB to its capturing of important past events, which eliminates spurious correlations for an event occurrence $e_k$ by predicting $Y_k$ based on $\mathcal{R}^k$ as stated in Proposition~\ref{prop:tgib}. 
\textbf{2)} TGIB achieved high prediction performance in politics networks such as USLegis and CanParl compared to the baselines. Specifically, in the inductive setting, USLegis achieved a $47.4\%$ performance improvement over the runner-up baseline, and CanParl achieved a $42.3\%$  improvement over the runner-up baseline. 
These findings suggest that in a political network, specific individuals or groups can have significant influence, and the GIB method effectively identifies these important people or groups and analyzes their roles and influence.
\textbf{3)} Additionally, TGIB achieved very high AP scores of $99.28\%$ and $99.68\%$ in the Wikipedia and Reddit datasets, respectively. Even though Wikipedia and Reddit already have high baseline performances, with runner-up performances of $98.28\%$ and $98.30\%$ respectively, TGIB achieves nearly perfect performance on both datasets.  TGIB is the only model that demonstrates performance exceeding $99\%$ on both datasets.

\begin{table*}[]
\centering
\caption{Explanation performance for TGIB and 6 baseline methods over 5 datasets.}
\label{table:explanation}
\setlength{\tabcolsep}{8pt}
\vspace{-3mm}
\resizebox{0.75\textwidth}{!}{
\begin{tabular}{lcccccc}
\toprule

              & \textbf{Wikipedia}  & \textbf{UCI}        & \textbf{USLegis}    & \textbf{CanParl}    & \textbf{Enron}       \\ \hline
Random        & 70.91 ± 1.03          & 54.51 ± 0.52          & 54.24 ± 1.34          & 51.66 ± 2.26          & 48.94 ± 1.28                \\
ATTN          & 77.31 ± 0.01          & 27.25 ± 0.01          & 62.24 ± 0.00          & 79.92 ± 0.01          & 68.28 ± 0.01                \\
Grad-CAM      & 83.11 ± 0.01          & 26.06 ± 0.01          & 78.98 ± 0.01          & 50.42 ± 0.01          & 19.93 ± 0.01                \\
GNNExplainer  & 84.34 ± 0.16          & 62.38 ± 0.46          & 89.42 ± 0.50          & 80.59 ± 0.58          & 77.82 ± 0.88                \\
PGExplainer   & 84.26 ± 0.78          & 59.47 ± 1.68          & \underline{91.42 ± 0.94}          & 75.92 ± 1.12          & 62.37 ± 3.82               \\
T-GNNExplainer & \underline{85.74 ± 0.56}          & \underline{68.26 ± 2.62}          & 90.37 ± 0.84          & \underline{80.67 ± 1.49}          & \underline{82.02 ± 1.94}               \\ \midrule
\cellcolor{grey}TGIB          & \cellcolor{grey}\textbf{88.09 ± 0.68} & \cellcolor{grey}\textbf{87.06 ± 1.04} & \cellcolor{grey}\textbf{93.33 ± 0.72} & \cellcolor{grey}\textbf{89.72 ± 1.18} & \cellcolor{grey}\textbf{83.55 ± 0.91}                  \\ \bottomrule
\end{tabular}}
\vspace{-1.5ex}
\end{table*}

\begin{figure*}[t]
  \centering
  \includegraphics[width=0.996\textwidth]{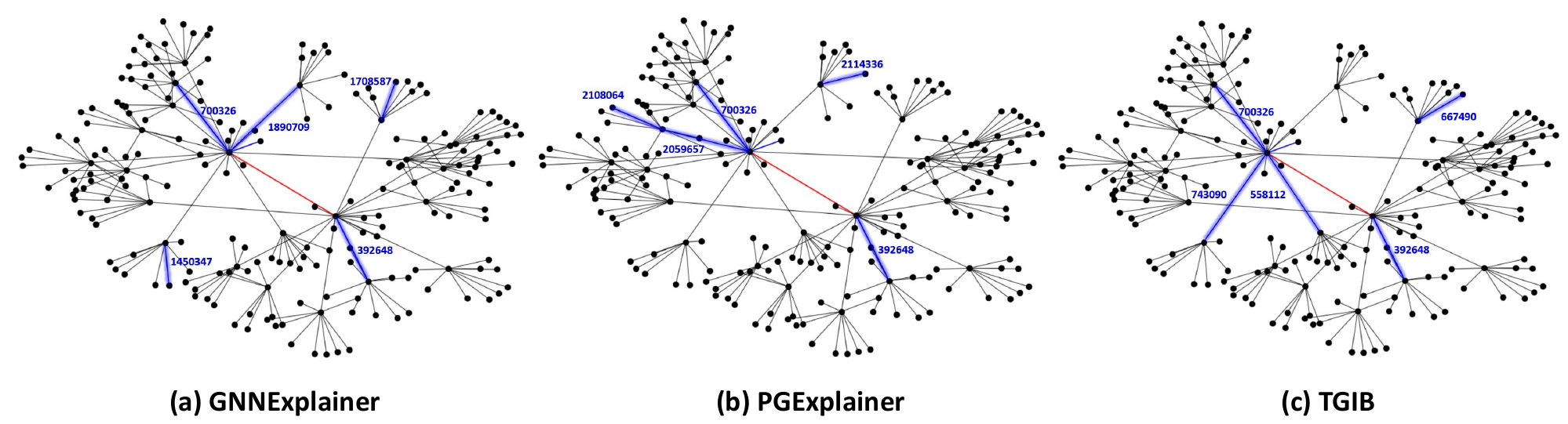}
  \vspace{-4mm}
  \caption{Comparison of explanation visualization for explanation models for static graphs and TGIB.}  
  \label{fig:visualization}
\end{figure*}

\subsection{Explanation Performance} \label{subsec:expl_performance}
\subsubsection*{\textbf{Baselines}}
We use six explanation methods, i.e. \textbf{Random}, \textbf{ATTN} \cite{velivckovic2017graph}, \textbf{Grad-CAM} \cite{pope2019explainability}, \textbf{GNNExplainer} \cite{ying2019gnnexplainer}, \textbf{PGExplainer}  \cite{luo2020parameterized}, \textbf{T-GNNExplainer} \cite{xia2022explaining} as baselines.
\textbf{Random} is the explanation results obtained by randomly sampling a number of nodes that satisfy sparsity. 
Details on baselines are presented in Appendix~\ref{apx:baseline}.
  
\subsubsection*{\textbf{Setup}}
We used TGAT~\cite{xu2020inductive} as the base model for all baselines that generate explanations in a post-hoc manner. 
Since the overall range of Fidelity used as an evaluation metric in TGNNExplainer has variability depending on the dataset or base model, calculating the area of the graph related to Fidelity may not provide consistent results.
Therefore, to evaluate the performance of explanations, we measure the proportion of predictions that match the model's original predictions when the model performs predictions based on explanations.  
In other words, predictions based on superior explanations have the same prediction label as predictions from the original graph. 
We also evaluate the explanation performance over various sparsity levels~\cite{xia2022explaining}, where the sparsity is defined as $|\mathcal{R}^k|/|\mathcal{G}^k|$.
Since higher levels of sparsity would always yield better explanation performance, we evaluate explanation performance in various sparsity environments.
We divide the sparsity level from 0 to 0.3 with intervals of 0.002, measure the performance of the explanation graph $\mathcal{R}^k$, and then calculate the area under the sparsity-accuracy curve. 
\subsubsection*{\textbf{Experiment Results.}}
We have presented the experimental results for explanation performance in Table~\ref{table:explanation}.
We have the following observations: \textbf{1)} TGIB outperforms all the baselines, indicating that it provides a high quality of explanation for the predictions. 
\textbf{2)} TGIB achieved up to 27.5\% improvement in performance on the UCI dataset compared to the baselines. Moreover, 2 out of 5 baselines (i.e., ATTN and Grad-CAM) show worse performance than random explanations. 
\textbf{3)} Some baselines, such as ATTN and Grad-CAM, fail to provide explanation, performing even worse than random explanations depending on the datasets, which significantly affect the reliability of the explainable model. However, TGIB demonstrates stable performance in explainablity, and its performance is further supported by a theoretical background. 
Furthermore, we present the inference time in Appendix~\ref{apx:efficiency}, demonstrating the efficiency of our method's evaluation while maintaining its superior explanations.

\vspace{-1mm}

\subsection{Explanation Visualization}  \label{subsec:visualization}
We compare the explanation visualizations for static graph explanation models, GNNExplainer and PGExplainer, with the temporal graph explanation model, TGIB, as shown in the Figure~\ref{fig:visualization}. In the figure, the target events are depicted by red solid lines, and the explanation events are depicted by blue solid lines. 
Additionally, we marked the difference in occurrence timestamps between the target event and each of the five explanation events.
Then, the mean time intervals are calculated as follows: (a) 1228523.4, (b) 1475006.2, and (c) 612333.2. This shows that the time interval for (c) is smaller than that for (a) and (b), and that the timestamps of the explanation events in TGIB are closer to the timestamps of the target events than those in GNNExplainer and PGExplainer.
In temporal graphs, events that occurred recently should have a greater influence on the target event compared to events that happened a long time ago; however, GNNExplainer and PGExplainer fail to capture this. In other words, it shows that GNNExplainer and PGExplainer cannot be easily generalized to temporal graphs because they don't capture temporal dynamics.
In contrast, TGIB considers the temporal information of the target event to generate explanations for temporal graphs.
Therefore, we can observe that TGIB is capable of capturing temporal dependencies along with graph topology.

\subsection{Ablation Study} \label{subsec:ablation}
\begin{table}[H]
\centering
\caption{Ablation study of the proposed components.}
\label{table:ablation}
\setlength{\tabcolsep}{6pt}
\vspace{-2mm}
\resizebox{0.98\columnwidth}{!}{
\begin{tabular}{lcccccc}
\toprule[1pt]
              & \textbf{USLegis}  & \textbf{UCI}        & \textbf{CanParl}    \\ \hline
w/o $X_{e_{k}}(t_k)$ \text{and} $Z_{e_{j}}(t_k)$                                        & 85.85 ± 0.36          & 86.10 ± 0.74          & 85.53 ± 0.62    \\
w/o $I(Y_k ; \mathcal{R}^k)$                      & 81.24 ± 0.41          & 78.30 ± 0.83          & 69.85 ± 0.74    \\
w/o $I(\mathcal{R}^k ; e_k,  \mathcal{G}^k)$      & 88.91 ± 1.14          & 90.45 ± 0.38          & 86.78 ± 0.37    \\ \midrule
\cellcolor{grey} \texttt{with all} (TGIB)                                          & \cellcolor{grey}\textbf{91.61 ± 0.34}          & \cellcolor{grey}\textbf{91.41 ± 0.41}          & \cellcolor{grey}\textbf{87.07 ± 0.44}    \\ \bottomrule[1pt]
\end{tabular}}
\end{table}

We conduct ablation studies to investigate the efficiency of the proposed model (i.e., TGIB).
Table~\ref{table:ablation} provides the ablation study of our proposed method based on link prediction. 
The \texttt{with all} setting indicates our final model including all components of TGIB.
We performed ablation studies on time-aware event representations (i.e., $X_{e_{k}}(t_k)$ and $ Z_{e_{j}}(t_k)$) and losses relevant to mutual information (i.e., $I(Y_k;\mathcal{R}^k)$ and $I(\mathcal{R}^k;e_k,\mathcal{G}^k)$).
We obtain the following observations: \textbf{1)} Capturing temporal information from event representations in the prediction based on the IB principle contributes to performance improvement. \textbf{2)} The model experiences declines in performance when the terms relevant to mutual information, $I(Y_k;\mathcal{R}^k)$ and $I(\mathcal{R}^k;e_k,\mathcal{G}^k)$, are not considered. Specifically, the removal of $I(Y_k;\mathcal{R}^k)$ results in $\mathcal{R}^k$ containing a large amount of label-irrelevant information, leading to difficulties in the final label prediction. Moreover, when $I(\mathcal{G}^k; e_k, \mathcal{R}^k)$ is not considered, it leads to a decrease in the generalization performance of the prediction due to the portions of $\mathcal{G}^k$ that have spurious correlations with $Y_k$, as mentioned in Section~\ref{subsec:inductivebias}.

\section{Conclusion} \label{conclusionsection}
In this work, we propose TGIB, a more reliable and practical explanation model for temporal graphs that can simultaneously perform prediction and explanation tasks.
The main idea is to provide time-aware explanations for the occurrence of the target events by restricting the flow of information from candidate events to predictions based on the IB theory.
We demonstrate that TGIB exhibits significant performance in both prediction and explanation across various datasets, and its explanation visualizations show that it can capture temporal relationships to extract important past events. 
Consequently, TGIB represents a more adaptable explainable model compared to existing models, capable of being efficiently trained on dynamically evolving graph settings without the exhaustive requirement for retraining from scratch, unlike other methods.

\noindent \subsubsection*{\textbf{Acknowledgement.}}
This work was supported by the National Research Foundation of Korea(NRF) grant funded by the Korea government(MSIT) (RS-2024-00335098), Institute of Information \& communications Technology Planning \& Evaluation (IITP) grant funded by the Korea government(MSIT) (No.2022-0-00157), and National Research Foundation of Korea(NRF) funded by Ministry of Science and ICT (NRF-2022M3J6A1063021).

\clearpage
\bibliographystyle{ACM-Reference-Format}
\balance
\bibliography{TGIB}

\clearpage
\twocolumn
\appendix


\section{Proof of Proposition~\ref{prop:tgib}} \label{apx:prop1}
\vspace{3mm}
In this section, we provide a detailed derivation process of equations corresponding to part of the proof of Proposition~\ref{prop:tgib}.

\noindent From Equation~\ref{eq:tgib}, the TGIB objective is 
\begin{equation} \label{eq:proof1}
 -I(Y_k;\mathcal{R}^k)+\beta I(\mathcal{R}^k;e_k,\mathcal{G}^k). 
\end{equation}
According to the chain rule of mutual information, we can decompose the term $\beta I(\mathcal{R}^k;e_k,\mathcal{G}^k)$ as follows:
\begin{equation} \label{eq:proof2}
-I(Y_k;\mathcal{R}^k) - \beta I(Y_k;e_k,\mathcal{G}^k |  \mathcal{R}^k) + \beta I(Y_k,\mathcal{R}^k ; e_k,\mathcal{G}^k).
\end{equation}

\noindent  Similarly, we decompose the term $-\beta I(Y_k;e_k,\mathcal{G}^k |  \mathcal{R}^k)$ from Equation~\ref{eq:proof2} as follows:

\begin{equation} \label{eq:proof3}
-I(Y_k;\mathcal{R}^k) - \beta I(Y_k;e_k|\mathcal{G}^k, \mathcal{R}^k) - \beta I(Y_k; \mathcal{G}^k|\mathcal{R}^k) + \beta I(Y_k,\mathcal{R}^k ; e_k,\mathcal{G}^k).
\end{equation}

\noindent  Since $\mathcal{R}^k$ is a subgraph of $\mathcal{G}^k$, $(\mathcal{G}^k, \mathcal{R}^k)$ holds no additional information over  $\mathcal{G}^k$. Therefore, Equation~\ref{eq:proof3} can be expressed as follows:

\begin{equation} \label{eq:proof4}
-I(Y_k;\mathcal{R}^k) - \beta I(Y_k;e_k|\mathcal{G}^k) - \beta I(Y_k; \mathcal{G}^k|\mathcal{R}^k) + \beta I(Y_k,\mathcal{R}^k ; e_k,\mathcal{G}^k).
\end{equation}

\noindent  We also follow the same process for the term $-I(Y_k;\mathcal{R}^k)$ as in Equation~\ref{eq:proof3} and \ref{eq:proof4}, as follows:

\begin{align}\label{eq:proof5} 
&-I(Y_k;\mathcal{G}^k,\mathcal{R}^k) +I(Y_k ;\mathcal{G}^k|\mathcal{R}^k)- \beta I(Y_k;e_k|\mathcal{G}^k) \nonumber \\ 
& \hspace{38mm} - \beta I(Y_k; \mathcal{G}^k|\mathcal{R}^k) + \beta I(Y_k,\mathcal{R}^k ; e_k,\mathcal{G}^k) \nonumber\\
&=-I(Y_k;\mathcal{G}^k) +I(Y_k ;\mathcal{G}^k|\mathcal{R}^k)- \beta I(Y_k;e_k|\mathcal{G}^k) \nonumber \\ 
& \hspace{38mm} - \beta I(Y_k; \mathcal{G}^k|\mathcal{R}^k) + \beta I(Y_k,\mathcal{R}^k ; e_k,\mathcal{G}^k). 
\end{align}

\noindent  Additionally, we decompose the term $I(Y_k,\mathcal{R}^k ; e_k,\mathcal{G}^k)$ according to the chain rule of mutual information as:

\begin{align}\label{eq:proof5} 
& -I(Y_k;\mathcal{G}^k) +I(Y_k;\mathcal{G}^k|\mathcal{R}^k)- \beta I(Y_k;e_k|\mathcal{G}^k) - \beta I(Y_k; \mathcal{G}^k|\mathcal{R}^k) \nonumber \\ 
& \hspace{40mm} + \beta I(\mathcal{R}^k;e_k,\mathcal{G}^k|Y_k)+\beta I(Y_k ; e_k,\mathcal{G}^k).
\end{align}

\noindent  We separate the terms related to $\mathcal{R}^k$ from those unrelated to $\mathcal{R}^k$ as follows:
\begin{align}\label{eq:proof6} 
& (1-\beta)I(Y_k ; \mathcal{G}^k|\mathcal{R}^k) + \beta I(\mathcal{R}^k;e_k,\mathcal{G}^k|Y_k) -I(Y_k;\mathcal{G}^k) \nonumber \\
& \hspace{40mm} + \beta I(Y_k;e_k,\mathcal{G}^k) - \beta I(Y_k;e_k|\mathcal{G}^k).
\end{align}

\noindent  We can substitute the terms unrelated to $\mathcal{R}^k$ with a constant $C$, since they remain constant regardless of $\mathcal{R}^k$ as:
\begin{equation} \label{eq:proof7}
(1-\beta)I(Y_k ; \mathcal{G}^k|\mathcal{R}^k) + \beta I(\mathcal{R}^k;e_k,\mathcal{G}^k|Y_k) + C.
\end{equation}

\noindent Therefore, we can utilize Equation~\ref{eq:proof7} to find $\mathcal{R}^k$ that minimizes Equation~\ref{eq:proof1} .

\section{Notations} \label{apx:notation}

In this section, we summarize the main notations used in this paper. Table~\ref{table:notations} provides the main notation and their descriptions.
\begin{table}[H]
\centering
\caption{Summary of the notations.}
\vspace{-3mm}
\label{table:notations}
\resizebox{1.07\columnwidth}{!}{
\begin{tabular}{ll}
\toprule[1pt]

\textbf{Notation}              & \textbf{Description}  \\\hline
$S=\{e_1, e_2, \cdots\}$          & Series of continuous events          \\
$t$          & Timestamp           \\
$\mathbf{G}=(\mathbf{V}, \mathbf{E})$          & Temporal graph with nodes $\mathbf{V}$ and edges $\mathbf{E}$           \\
$f$          & Self-explainable temporal graph model           \\
$u, v$          & User and item node, respectively          \\
$e_i=(u_i,v_i,t_i,{att}_i)$        &  Event $e_i$ between node $u_i$ and $v_i$ at time $t_i$ with attribute ${att}_i$         \\
$\mathbf{G}^k$          & Graph constructed immediately before the timestamp $t_k$           \\
$\mathcal{G}^k$          & L-hop computation graph of $e_k$         \\
$\mathcal{R}^k$          & Subgraph of $\mathcal{G}^k$ considered as important events (i.e., explanation).         \\
${Y}$      & Ground truth label          \\
$\hat{Y}$      & Label prediction          \\
$I(\cdot , \cdot)$  & Mutual information function          \\
$\mathcal{N}(z; t)=\{z_1,\cdots,z_n\}$   &  Neighboring nodes for node $z$ at time $t$, where $n$ is \# of neighbors         \\
$att_{z,i}$ & Attribute of an interaction between $z$ and $z_i$       \\ 
$h_z^{(l)}(t)$          & Representation of node $z$ at time $t$ in the $l$-th layer \\ 
$d$          & Dimension of the node representation \\ 
$x_z$          & Raw feature of node $z$ \\ 
$f_{\text{node}}$          & Dimension of the raw node feature \\ 
$f_{\text{edge}}$          & Dimension of the raw edge feature \\ 
$\Phi_{d_T}$          & Time encoding function \\ 
$d_T$          & Output dimension of a function $\Phi_{d_T}$ \\ 
$e_k$          & Target event \\ 
$e_j$          & Candidate event \\ 
$e_k^{neg}$          & Negative sample event for $e_k$ \\ 
$n_k$          & Negative sample node for $e_k$ \\ 
$X_{e'}(t_k)$          & Time-aware representation for the target event $e'$ at time $t_k$ \\ 
$Z_{e_{j}}(t_k)$          & Time-aware representation for the candidate event $e_j$ at time $t_k$ \\ 
$\tilde{Z}_{e_{j}}(t_k)$          & Valid event representation for the candidate event $e_j$ at time $t_k$ \\ 
$p'_j$          & Probability of $e_j$ given $e'$, respectively \\ 
$g$          & MLP function calculating the probability $p'_j$ \\ 
$\alpha_e' \sim Bern(r)$          & Mask for $e$ sampled from $Bernoulli(r)$ \\ 
$H^k$          & Representation of the explanation graph $\mathcal{R}^k$ \\ 
$q_\theta$          & Link predictor \\ 
\bottomrule[1pt]
\end{tabular}}
\end{table}

\section{Efficiency Evaluation} \label{apx:efficiency}

\begin{table}[H]
\centering
\caption{Inference time of one explanation.}
\vspace{-3mm}
\label{table:inference_time}
\begin{tabular}{lcc}
\toprule[1pt]
              & \textbf{Wikipedia} & \textbf{Reddit} \\ \hline
TGAT + ATTN          & 0.02               & 0.04            \\
TGAT + Grad-CAM      & 0.04               & 0.06            \\
TGAT + GNNExplainer  & 8.52               & 11.03           \\
TGAT + PGExplainer   & 0.10               & 0.11            \\
TGAT + TGNNExplainer &  25.36             &   81.57       \\ \midrule
\rowcolor{grey}TGIB          & 0.11               &  0.53               \\ \bottomrule[1pt]
\end{tabular}
\end{table}

In this section, we measure the inference time to produce an explanation to investigate the effectiveness of TGIB. 
Table~\ref{table:inference_time} shows the inference time for several explanation models on Wikipedia and Reddit datasets.
All baselines are models that generate explanations in a post-hoc manner, and TGAT is used as their base model. The inference time is calculated for all test events. We obtain the following observations: \textbf{1)} TGNNexplainer is time-consuming due to its reliance on the MCTS algorithm. On the other hand, TGIB efficiently detects important candidate events based on the IB principle by injecting stochasticity into past candidate events to generate explanations.
\textbf{2)} Some explanation models for static graphs have fast inference times, but according to Table~\ref{table:explanation}, they exhibit low-quality explanations for temporal graphs. However, although TGIB is slower than them, it does not demonstrate a significant time difference with them and achieve notable explanation performance.
Therefore, considering both the quality of explanations and the inference time, TGIB is regarded as a reasonable model.

\section{Algorithm} \label{apx:alg}
\vspace{-2.5ex}

\begin{algorithm}[h]
        \small
        \DontPrintSemicolon
        \SetAlgoLined
    \caption{Overview of TGIB training}
    \label{alg:algorithm2}
    \KwInput{
    Temporal graph $S=\{e_1,e_2,\cdots\}$ where $e_i=\{u_i,v_i,t_i,{att}_i\}$, The number of epochs $T$
    }
    \textsf{Training set} $S_{train}\leftarrow\{e_i=\{u_i,v_i,t_i,{att}_i\}\mid t_i<0.7T\}$ \\
    \For{\textsf{epoch} \textbf{in} \textsf{1,2,$\cdots$,T}}
    {
        \For{$k$ \textbf{in} \textsf{1,2,$\cdots$,} $\lvert S_{train}\rvert$}
        {
            \tcc{Calculate an event representation}
            $X_{e_k}=[h_{u_k}(t_k)\mathbin\Vert h_{v_k}(t_k)\mathbin\Vert \Phi_{d_T}(0)\mathbin\Vert \textrm{att}_k]$ \\
            \tcc{Calculate a negative event representation}
            $n_k \leftarrow \text{Sample random } v_i \text{ from } S_{train} \text{ with } v_i \neq v_k$ \\
            $X_{e_k^{neg}}=[h_{u_k}(t_k)\mathbin\Vert h_{n_k}(t_k)\mathbin\Vert \Phi_{d_T}(0)\mathbin\Vert \textrm{att}_k ]$ \\
            \tcc{Extract an L-hop computational graph}
            $\mathcal{G}^k \leftarrow \{e_j \mid e_j \text{ is L-hop from } e_k\}$ \\
            \For{$e_j$ \textbf{in} $\mathcal{G}^k$} 
            {
                \tcc{Calculate a candidate event representation}
                $Z_{e_{j}}(t_k) = [  h_{u_{j}}(t_j) \mathbin\Vert  h_{v_{j}}(t_j)  \mathbin\Vert \Phi_{d_{T}}(t_k-t_j)  \mathbin\Vert \textrm{att}_j ]$ \\
                \tcc{Calculate a probability of $e_j$ given $e_k$}
                $p_j^k = p(e_j\mid e_k,  \mathcal{G}^k) = \sigma ( g (X_{e_{k}}(t_k),   Z_{e_{j}}(t_k)  ))$ \\
                \tcc{Evaluate the mutual information loss}
                $\mathcal{L}_{\textrm{MI}} = \mathbb{E}_{p(e_k, \mathcal{G}^k)}[ \sum_{e_j\in\mathcal{G}^k} p_j^k \log \frac{p_j^k}{r} + (1-p_j^k) \log \frac{1-p_j^k}{1-r} ]$ \\
                \tcc{Extract a valid event representation and $\mathcal{R}^k$}
                $\tilde{Z}_{e_j}(t_k) = \alpha_j^k Z_{e_j}(t_k), \quad \alpha_j^k \sim Ber(p_j^k)$ \\
                $\mathcal{R}^k \leftarrow e_j \text{ if } \alpha_j^k = 1$\\
            }
        \tcc{Calculate an representation of $\mathcal{R}^k$}
        ${H}_{+}^k = \textrm{Readout} [  \{ \tilde{Z}_{e_j}(t_k) \mid e_j \in \mathcal{G}^k \} ]$ \\
        Derive ${H}_{-}^k$ based on $X_{e_k^{neg}}$ in the same way as $X_{e_k}$ \\
        \tcc{Evaluate the link prediction loss}
        $\mathcal{L}_{\textrm{cls}} = \sum_{e_k \in S} -\log [ \sigma( q_{\theta}( X_{e_k}, {H}^k_{+} ) )] - N\cdot\mathbb{E}_{\text{neg} \sim P_n} \log [ \sigma(q_{\theta} ( X_{e_k^\text{neg}}, {H}^k_{-} ) ) ]$ \\
        \tcc{Calculate a total loss and update the model}
        $\mathcal{L}=\mathcal{L}_\textrm{cls} + \mathcal{L}_\textrm{MI}$ \\
        Update model parameters by gradient descent
        }
    }  
\end{algorithm}

\section{Baseline Details \& Source code} \label{apx:baseline}
Details on compared baselines for each experiments are provided below (Sec~\ref{apx:baseline:link},~\ref{apx:baseline:expl}). 
Official source codes for the baselines are also provided in table~\ref{tbl:source}.

\subsection{Link Prediction} \label{apx:baseline:link}
\begin{itemize}[leftmargin=4mm]
  \item \textbf{Jodie} \cite{kumar2019predicting} utilizes two coupled RNNs to produce temporal user and item embeddings then projects temporal user embeddings at a future time to predict the future user-item interaction.
  \item \textbf{DyRep} \cite{trivedi2019dyrep} is an RNN-based method that propagates messages along the interaction, where messages are collected from neighbors of nodes involved in the interaction. 
  \item \textbf{TGAT} \cite{xu2020inductive} applies Self-Attention Mechanism (SAM) to model both spatial and temporal relationships concurrently by using features with functional time encodings.
  \item \textbf{TGN}  \cite{rossi2020temporal} uses both RNN-based model and SAM-based model architecture, where the former is used to update node memory for modeling temporal dependencies and the latter is used to compute node embeddings modeling both spatial and temporal information similar to TGAT.
  \item \textbf{TCL} \cite{wang2021tcl} is a SAM-based method for modeling both spatial and temporal dependencies, where contrastive objective function is used to optimize encoders.
  \item \textbf{CAW-N} \cite{wang2021inductive} utilizes Causal Anonymous Walks (CAWs) extracted from temporal random walks to represent temporal graph dynamics with anonymized node identities and the RNN for encoding CAWs.
  \item \textbf{GraphMixer} \cite{cong2023we} uses three modules, 1) MLP-based link encoder, 2) node encoder using neighbor mean-pooling and 3) MLP-based link classifier without using commonly utilized RNN or SAM-based architecture in temporal modeling.
\end{itemize}


\subsection{Explanation Performance} \label{apx:baseline:expl}
\begin{itemize}[leftmargin=4mm]
    \item \textbf{ATTN} \cite{velivckovic2017graph} utilizes learned attention weights of GAT as edge importances to make explanations.
    \item \textbf{Grad-CAM} \cite{pope2019explainability} provides post-hoc explanation for the GNN by using importance score computed with gradients of the logit value with respect to the node embeddings.
    \item \textbf{GNNExplainer} \cite{ying2019gnnexplainer} is a post-hoc explanation model for providing explanations for GNN predictions. Specifically, this model learns to mask the input graph while including label information in the detected subgraphs.
    \item \textbf{PGExplainer}  \cite{luo2020parameterized} uses parameterized edge distributions obtained by explanation network to make explanation subgraph. The explanation network is optimized by maximizing the mutual information between the explanatory subgraph and predictions of the GNN.
    \item \textbf{T-GNNExplainer} \cite{xia2022explaining} aims to explain TGNNs post-hoc, employing both navigator and explorer components. The pretrained navigator captures inductive relationships among events, and the explorer seeks the optimal combination of candidates for explanation.
\end{itemize}

\begin{table}[H]
\caption{Source code links of the baseline methods}
\vspace{-3mm}
\label{tbl:source}
\resizebox{\columnwidth}{!}
{
{
\renewcommand{\arraystretch}{1.0}
\begin{tabular}{c|p{7cm}} \toprule[1pt]
\textbf{Methods} & \textbf{Source code} \\ \hline
Jodie \cite{kumar2019predicting}       & \url{https://github.com/claws-lab/jodie}           \\ \hline
DyRep \cite{trivedi2019dyrep}       & \url{ https://github.com/hunto/DyRep}           \\ \hline
TGAT \cite{xu2020inductive}       & \url{https://github.com/StatsDLMathsRecomSys/Inductive-representation-learning-on-temporal-graphs}           \\ \hline
TGN \cite{rossi2020temporal}      & \url{https://github.com/twitter-research/tgn}           \\ \hline
CAW-N \cite{wang2021inductive}      & \url{https://github.com/snap-stanford/CAW}           \\ \hline
GraphMixer \cite{cong2023we}      & \url{https://github.com/CongWeilin/GraphMixer}  \\ \hline
Grad-CAM \cite{pope2019explainability}      & \url{https://github.com/ppope/explain_graphs}  \\ \hline
GNNExplainer \cite{ying2019gnnexplainer}      & \url{https://github.com/RexYing/gnn-model-explainer}  \\ \hline
PGExplainer \cite{luo2020parameterized}      & \url{https://github.com/flyingdoog/PGExplainer}  \\ 
\bottomrule[1pt]
\end{tabular}
}
}
\end{table}






\end{document}